\documentclass[letterpaper]{article} %
\usepackage{aaai2026}  %
\usepackage{times}  %
\usepackage{helvet}  %
\usepackage{courier}  %
\usepackage[hyphens]{url}  %
\usepackage{graphicx} %
\usepackage{natbib}  %
\usepackage{caption} %
\usepackage{newfloat}
\usepackage{listings}
\DeclareCaptionStyle{ruled}{labelfont=normalfont,labelsep=colon,strut=off} %
\nocopyright %

\title{Algorithms for Deciding the Safety of States in Fully Observable Non-deterministic Problems: Technical Report}
\author{
    Johannes Schmalz,
    Chaahat Jain
}
\affiliations{
    Saarland University, Germany
}

\newcommand{\report}[1]{#1}

\newcommand{\short}[1]{}

\usepackage{xparse}

\newcommand{\argmin}{\operatornamewithlimits{argmin}}

\newcommand{\lts}{\Theta}

\newcommand{\transitions}{\ensuremath{\mathcal T}}

\newcommand{\traces}{\textit{Paths}}
\newcommand{\maxTraces}{\textit{MaxPaths}}

\newcommand{\state}{\ensuremath{s}}
\newcommand{\succStateOf}[1]{\ensuremath{ {#1}' }}
\newcommand{\succState}{\ensuremath{ \succStateOf{\state} }}

\newcommand{\policy}{\ensuremath{\pi}}
\newcommand{\policies}{\Pi}
\DeclareDocumentCommand{\policyLTS}{G{\policy}}{\lts^{#1}}
\DeclareDocumentCommand{\policyTransitions}{G{\policy}}{\transitions^{#1}}
\DeclareDocumentCommand{\policyTraces}{G{\policy}}{\traces^{#1}}
\DeclareDocumentCommand{\policyMaxTraces}{G{\policy}}{\maxTraces^{#1}}
\DeclareDocumentCommand{\policyDistance}{G{s}}{\delta_{#1}}
\DeclareDocumentCommand{\proximalPolicies}{G{\radius}G{\policy}D(){s}}{\policies_{#2, #1}^{#3}}

\DeclareDocumentCommand{\Jstar}{G{\radius}G{\policy}}{J^*_{#2, #1}}
\DeclareDocumentCommand{\unsafetyReward}{G{\radius}}{\mathsf{F}_{#1}}
\DeclareDocumentCommand{\devCost}{G{\radius}G{\policy}}{\mathsf{C}_{#2, #1}}
\DeclareDocumentCommand{\bellman}{G{\radius}G{\policy}}{\mathcal{B}_{#2, #1}}

\newcommand{\radius}{r}

\newcommand{\optVI}{\ensuremath{\text{Prop-}\mathcal{U}}\xspace}

\newcommand{\nPI}{\ensuremath{\text{nPI}}\xspace}
\newcommand{\iPI}{\ensuremath{\text{iPI}}\xspace}
\newcommand{\optPI}{\iPI}
\newcommand{\reciPI}{\ensuremath{\text{rec-}\iPI}}
\newcommand{\MarkUnsafe}{\ensuremath{\textit{Mark-Unsafe}}\xspace}
\newcommand{\Greedy}{\ensuremath{\textit{Greedy}}\xspace}
\newcommand{\pInit}{\ensuremath{\p_{\text{init}}}\xspace}
\newcommand{\TarjanSafe}{\ensuremath{\text{TarjanSafe}}\xspace}
\newcommand{\orderApp}{\ensuremath{\textit{app}}\xspace}
\newcommand{\orderLearn}{\ensuremath{\textit{learn}}\xspace}
\newcommand{\orderRand}{\ensuremath{\textit{rand}}\xspace}
\newcommand{\iPIApp}{\ensuremath{\iPI(\orderApp)}\xspace}

\newcommand{\actionOrder}{\ensuremath{\mathbb{A}}\xspace}

\newcommand{\vect}[1]{\ensuremath{\bm{#1}}}

\usepackage{graphicx} %
\usepackage[tight-spacing=true]{siunitx} %
\usepackage[export]{adjustbox}
\usepackage{lscape}
\usepackage{makecell}
\usepackage{xspace}
\usepackage{fwt_math_and_notation}
\usepackage{fwt_lp}
\usepackage{cuted}
\usepackage{framed}
\usepackage{bm}

\usepackage{tikz-network}
\usepackage{tikz}
\usetikzlibrary{fit, calc}
\usetikzlibrary{shapes.multipart}  %

\usepackage[ruled,vlined,linesnumbered, commentsnumbered]{algorithm2e} %
\SetKwProg{function}{Function}{}{}
\SetKwProg{class}{Class}{}{}
\SetKwRepeat{ForIUntil}{for \(i \in \{0,1,\dots\}\)}{until}
\SetKwBlock{RepeatForever}{repeat}

\robustify{\bfseries}
\usepackage{enumerate}

\usepackage[inline]{enumitem}

\usepackage{multirow}

\usepackage{cleveref}
\crefname{algorithm}{alg.}{algs.}
\crefname{line}{line}{lines}
\crefname{def}{defn.}{defns.}
\crefname{lem}{lem.}{lem.s}
\crefname{theorem}{thm.}{thm.s}
\crefname{figure}{fig.}{figs.}
\crefname{table}{tab.}{tabs.}
\crefname{teo}{thm.}{thm.s}
\crefname{assump}{assump.}{assump.s}

\newcommand{\wrt}{\ensuremath{\text{w.r.t.\ }}}
\newcommand{\st}{\ensuremath{\text{s.t.\ }}}

\newcommand{\fond}{\ensuremath{\Theta}\xspace}
\newcommand{\transit}{\ensuremath{\mathcal{T}}\xspace}
\newcommand{\failStates}{\ensuremath{\Ss_{f}}\xspace}
\newcommand{\SZ}{\ensuremath{\Ss_{I}}\xspace}

\newcommand{\iverson}[1]{\ensuremath{\llbracket #1 \rrbracket}}

\newcommand{\lao}{LAO$^*$\xspace}

\newcommand{\lrtdp}{LRTDP\xspace}

\newcommand{\supp}{\ensuremath{\text{supp}}}

\newcommand{\V}{\ensuremath{V}\xspace}
\newcommand{\Q}{\ensuremath{Q}\xspace}
\newcommand{\Qsa}{\ensuremath{Q(\s, \ac)}\xspace}
\newcommand{\qvalue}{\Q-value\xspace}
\newcommand{\qvalues}{\Q-values\xspace}

\newcommand{\paths}{\ensuremath{\textit{Paths}}}

\begin{document}

\maketitle

\begin{abstract}
Learned action policies are increasingly popular in sequential decision-making, but suffer from a lack of safety guarantees.
Recent work introduced a pipeline for testing the safety of such policies under initial-state and action-outcome non-determinism.
At the pipeline's core, is the problem of deciding whether a state is safe (a safe policy exists from the state) and finding \emph{faults}, which are state-action pairs that transition from a safe state to an unsafe one.
Their most effective algorithm for deciding safety, \TarjanSafe, is effective on their benchmarks, but we show that it has exponential worst-case runtime with respect to the state space.
A linear-time alternative exists, but it is slower in practice.
We close this gap with a new \emph{policy-iteration} algorithm \optPI, that combines the best of both: it matches \TarjanSafe's best-case runtime while guaranteeing a polynomial worst-case.
Experiments confirm our theory and show that in problems amenable to \TarjanSafe \optPI has similar performance, whereas in ill-suited problems \optPI scales exponentially better.
\end{abstract}

\short{
\noindent
\textbf{Extended Version --- (TODO: cite)}
\red{(IMPORTANT: turn copyright back on)}
}

\begin{links}
  \link{Short Version Published at ICAPS 2026}{https://doi.org/10.1609/icaps.v36i1.42858}
\end{links}

\section{Introduction}
Learned action policies are gaining traction in AI and AI planning, e.g., \citet{mnih:etal:nature-15,silver:etal:nature-16,silver:etal:science-18,Toyer2020:ASNets}.
However, such policies come without any safety guarantees, creating demand for safety assurance methods.
We focus on the safety-testing framework of \citet{Jain2025:faults} (henceforth JEA) in fully observable non-deterministic (FOND) planning.
In JEA's setting, a policy $\pi$ is \textbf{safe} in a given state $s$ if its execution will never lead to a \textbf{failure condition} $\phi_F$, a state is \textbf{safe} if it has a safe policy, and \textbf{faults} are state-action pairs $(s, \pi(s))$ where $s$ is safe but applying $\pi(s)$ non-deterministically leads to an unsafe state $s'$.
Their framework takes a learned policy, finds paths therein that lead to fail states, and then finds faults by deciding the safety of states along the paths.
These faults can be used to repair the policy.
To decide whether a state is safe JEA consider two orthogonal approaches: (1) adapting well-known algorithms for Markov Decision Processes (MDPs), such as Value Iteration (VI), \lao~\cite{Hansen2001:ilao}, and \lrtdp~\cite{Bonet2003:lrtdp}; (2) creating a specialised Depth-First Search (DFS) with a mechanism for detecting Strongly Connected Components (SCCs), which they call \TarjanSafe.
Their results showed that \TarjanSafe is very effective for deciding safety.

In this paper, we show that \TarjanSafe has an \emph{exponential worst-case time complexity \wrt the state space} using a pathological family of tasks where \TarjanSafe expands the same states many times. %
Despite this limitation, \TarjanSafe performs well in practice because it often explores only a small fraction of the state space.
To contrast, we present \optVI, which follows techniques from reachability games~\cite{Diekert2022:reachability-games} and works by propagating unsafe states.
Its worst case is linear, but it is unusable in practice because there are too many unsafe states to propagate.
\optVI is unusable in practice because there are too many unsafe states to propagate.
To address this gap we introduce \optPI, a new algorithm based on \textbf{policy iteration} that combines a \emph{polynomial worst-case runtime} while enjoying the practical speed of \TarjanSafe;
\optPI starts with an initial policy \(\pInit\) and attempts to prove its safety, making minimal changes to the policy as unsafe states are discovered; this makes \optPI very efficient at finding safe policies similar to \(\pInit\), which occurs frequently.
Finally, we confirm our theory experimentally.
We use JEA's benchmarks, which mostly correspond to \TarjanSafe's best case; there, \optPI shows similar performance to \TarjanSafe.
Then, we introduce the \textit{Flappy Bird} domain to showcase a problem ill-suited to \TarjanSafe, and indeed \optPI scales exponentially better.
This places \optPI as the current best algorithm for deciding state safety and gives valuable insight into what makes a good algorithm for this problem.

Our contributions are: a simplified formalism for FOND safety, new algorithms for deciding safety (\optVI and \optPI), a complexity analysis of \TarjanSafe and our algorithms, and an empirical evaluation that shows \optPI's effectiveness.

\section{Background}

We consider fully-observable non-deterministic (FOND) planning tasks \cite{Daniele2000:FOND,Cimatti2003:FOND}.
We follow the formalism of JEA, but break it down into a simpler, equivalent formulation.
A task's transition system is defined by \(\fond = \langle \Ss, \A, \transit, \SZ \rangle\) where
\short{
\Ss is a finite set of states, \A is a finite set of actions and we write \(\A(\s)\) to denote the set of actions applicable in state \s, \transit is a non-deterministic transition function where \(\transit(\s,\ac) \in \mathcal{P}(\Ss) \setminus \emptyset\) gives the set of states that may occur after applying \ac to \s, \(\SZ \subsetneq \Ss\) is a non-empty set of initial states.
}
\report{
\begin{itemize}
\item \Ss is a finite set of states,
\item \A is a finite set of actions and we write \(\A(\s)\) to denote the set of actions applicable in state \s,
\item \transit is a non-deterministic transition function where \(\transit(\s,\ac) \in \mathcal{P}(\Ss) \setminus \emptyset\) gives the set of states that may occur after applying \ac to \s,
\item \(\SZ \subsetneq \Ss\) is a non-empty set of initial states.
\end{itemize}
}

\textbf{Policies} are partial functions \(\p : \Ss \to \A\) with the property that \(\p(\s) \in \A(\s)\) wherever defined.
A policy \p induces a \textbf{policy graph} \(\fond^{\p} = \langle \Ss, \A, \transit^{\p}, \SZ \rangle\) with the same components as \fond except \(\transit^{\p}\) has \(\transit^{\p}(\s,\ac) = \transit(\s,\ac)\) if \(\ac = \p(\s)\) and is undefined otherwise.
A \textbf{path} is a finite or infinite sequence of alternating states and actions \(\sigma = \s_0, \ac_0, \s_1\dots\) that satisfies \(\ac_{i} \in \A(\s_i)\) and \(\s_{i+1} \in \transit(\s_i, \ac_i)\) for all \(i \geq 0\) in \(\sigma\).
The set of all paths from a state \s is \(\paths(\s)\) and \(\paths^{\p}(\s)\) is the set of paths from \s in the policy graph \(\fond^{\p}\).

Typically, FOND tasks are also specified with a non-empty set of goal states \(\Sg \subseteq \Ss\).
We call such tasks \(\fond = \langle \Ss, \A, \transit, \SZ, \Sg \rangle\) \textbf{goal-reaching FOND tasks}.
\short{
To solve them we consider strong cyclic solutions, which we present as policies \p \st following \p from any initial state \(\sZ \in \SZ\) will eventually reach a goal in \Sg assuming \emph{fairness}; fairness ensures that each effect in \(\transit(\s,\ac)\) is eventually triggered if \ac is applied to \s often enough~\cite{Cimatti2003:FOND}.
}
\report{
To solve them we consider strong cyclic solutions, which we present as policies \p \st following \p from any initial state \(\sZ \in \SZ\) will eventually reach a goal in \Sg assuming \emph{fairness}.

\begin{assumption}[Fairness]\label{assump:fairness}
If an action \ac is applied infinitely-many times in state \s, then each non-deterministic effect in \(\transit(\s,\ac)\) must trigger infinitely often~\cite{Cimatti2003:FOND}.
\end{assumption}

Formally, \p solves a goal-reaching FOND task if all finite-length paths in \(\bigcup_{\sZ \in \SZ} \paths^{\p}(\sZ)\) terminate in \Sg.
}

Recall that our setting is to identify whether a safe policy exists for states in a learned policy's envelope, so in this paper we focus on avoiding certain fail states rather than goal-reaching.
Thus, we specify a non-empty set of \textbf{fail states} \(\failStates \subseteq \Ss\), giving us \textbf{state-avoiding FOND tasks} \(\fond = \langle \Ss, \A, \transit, \SZ, \failStates \rangle\).\footnote{JEA use a failure condition \(\phi_{F}\), but for our discussion it is equivalent to consider a specified set of fail states \(\failStates \subseteq \Ss\).}
A solution to such \fond is a policy \p that never encounters any fail state, i.e., for all finite and infinite paths \(\sigma \in \bigcup_{\sZ \in \SZ} \paths^{\p}(\sZ)\), \(\sigma\) does not contain any states in \failStates.
State-avoiding FOND lets us express notions of safety, i.e., that a ``bad'' state should never happen, e.g., a self-driving car must never crash into a wall.
A state \s is called \textbf{unsafe} if no policy can guarantee that it avoids fail states from \s, i.e., \(\forall \p \; \exists \sigma \in \paths^{\p}(\s) : \sigma \cap \failStates \neq \emptyset\).
The set of unsafe states is called the unsafe region \(\mathcal{U}\).

\short{
State-avoiding tasks can be transformed into goal-reaching ones, so state-of-the-art algorithms for state-reaching FOND can be applied to our setting, e.g., PR2~\cite{Muise2024:PR2}.
However, we limit the scope of this paper to fundamental algorithms that are native to state-avoiding, letting us make a fairer comparison to JEA and perform sharper complexity analyses, thereby establishing a foundation for future work.
Moreover, using their algorithms in our codebase is non-trivial because it requires a translation from JANI to variants of PDDL, which is an open line of research, e.g., \cite{Klauck2018:JaniToPDDL}.
}

\report{

Whether a policy \p is state-avoiding (i.e., safe) can be evaluated with the following \textit{value function}.
\begin{definition}\label{def:state-avoid-value-function}
The value of policy \p at state \s is
\[
\V^{\p}(\s) = \begin{cases}
1 &\text{ if } \s \in \failStates \\
\max_{\s' \in \transit(\s, \p(\s))} \V^{\p}(\s') &\text{ if } \s \not\in \failStates.
\end{cases}
\]
\end{definition}
The policy \p is safe from a state \s, i.e., there is a way to avoid unsafe states, iff \(\V^{\p}(\s) = 0\).

\(\V^*(\s)\) is the optimal value at \s, given by \(\V^*(\s) = \min_{\p \in \Pi} \V^{\p}(\s)\).
Given \(\V^*\), if a safe policy exists from \s, then one can extract a safe policy by following \(\V^*\) greedily.
\(\V^*\) can be computed with Value Iteration (VI).
VI starts with an initial value function, e.g., \(\V(\s) = \iverson{ \s \in \failStates } \; \forall \s \in \Ss\), and then applies Bellman backups:
\[\V(\s) \gets \min_{\ac \in \A(\s)} \max_{\s' \in \transit(\s,\ac)} \V(\s').\]
The term being minimised in the Bellman backup is often called a \qvalue, so for us \(\Qsa = \max_{\s' \in \transit(\s,\ac)} \V(\s')\).
By repeatedly applying Bellman backups over all states the value function will eventually converge to \(\V^*\).
JEA give arguments for this, and it is also demonstrated by \optVI later.

Dually to \(\V^*(\s) = 0\) indicating a safe state, \(\V^*(\s) = 1\) indicates that \s is an unsafe state, from which there is no policy that can guarantee that it avoids fail states \failStates.

}

Now, we return to the problem formalisation.
JEA look for safe policies in tasks that have goals, other terminal states, and fail states to avoid.
Their policies must avoid the fail states, either by reaching a terminal state and stopping there, or by cycling in a way that avoids the fail states.
JEA's tasks can be transformed into our state-avoiding tasks by adding a self-loop action to goal and terminal states, allowing policies to avoid fail states indefinitely by using these self loops rather than terminating.
For the rest of this paper we focus on deciding the safety of individual states, so we only consider singleton initial states \(\SZ = \{\sZ\}\).

\report{
We note that state-avoiding tasks can be cast into goal-reaching ones.
Given a state-avoiding problem \(\fond = \langle \Ss, \A, \transit, \SZ, \failStates \rangle\), we can define an equivalent goal-reaching problem \(\fond' = \langle \Ss \cup \{g\}, \A, \transit', \SZ, \{ g \} \rangle\).
This introduces the artificial goal \(g\) and all transitions from non-fail states gain \(g\) as a non-deterministic effect from all safe states; also, we turn fail states into ``traps'' that only loop back to themselves and can never reach the goal.
Formally, \(\forall \s \in \Ss, \ac \in \A(\s)\)
\[
\transit'(\s,\ac) = \begin{cases}
\transit(\s,\ac) \cup \{g\} &\text{ if } \transit(\s,\ac) \text{ defined and } \s \not\in \failStates \\
\{\s\} &\text{ if } \transit(\s,\ac) \text{ defined and } \s \in \failStates \\
\text{undefined} &\text{ if } \transit(\s,\ac) \text{ undefined.}
\end{cases}
\]
Now, \(\fond'\) has a policy that always reaches its goal \(\{g\}\) iff \(\fond\) has a policy that avoids the states \failStates.
This is clear because a goal-reaching policy for \(\fond'\) can not enter any fail state \failStates, and therefore describes a state-avoiding policy for \fond.
Going the other direction, a state-avoiding policy for \fond never enters \failStates, and so all its actions have \(g\) as one of the effects, which means that, due to the fairness assumption~(\cref{assump:fairness}), the policy must eventually reach the artificial goal in \(\fond'\).

This transformation from state-avoiding to goal-reaching lets us solve the state-avoiding problem in terms of goal-reaching with state-of-the-art FOND planners such as PR2~\cite{Muise2024:PR2}.
However, we limit the scope of this paper to fundamental algorithms that are native to state-avoiding, letting us make a fairer comparison to JEA and perform sharper complexity analyses, thereby establishing a foundation for future work.
Moreover, using their algorithms in our codebase is non-trivial because it requires a translation from JANI to variants of PDDL, which is an open line of research, e.g., \cite{Klauck2018:JaniToPDDL}.
}

\section{Algorithms}\label{sec:algos}

We first explain \TarjanSafe~(JEA) and prove that it has a worst-case exponential runtime \wrt the state-space, but a good best-case that reflects its strong performance in practice.
Then, we show \optVI, an algorithm that solves state-avoiding tasks in linear time, but has a worse best-case that makes it impractical.
Finally, we introduce a Policy Iteration algorithm and prove that it combines a polynomial worst case with the same best case as \TarjanSafe.
We use the notation \(m = |\A|\), \(n = |\Ss|\), \(b = \max_{\s \in \Ss, \ac \in \A(\s)} |\mathcal{T}(\s,\ac)|\) denotes a bound on the branching factor of non-determinism, and \(\p_{\text{min-safe}}\) is a safe policy with a minimal number of states in its policy graph.
We assume that \(\A(\s) \neq \emptyset\) for each state, so \(m \geq n\).
Our complexity results are summarised in \cref{tab:algo-complexities}.
We restrict the best-case runtime analyses to tasks where a safe policy exists; if no safe policy exists then all the algorithms we consider can be implemented with an early stop, in which case they terminate in one or two steps, i.e., \(O(1)\).

\begin{table}[t!]
\center
\begin{tabular}{lll}
                                                                              & Best case                           & Worst case         \\ \hline
\TarjanSafe~\report{(\cref{alg:tarjansafe})}                                    & \(\Theta(|\p_{\text{min-safe}}|)\)  & \(\Omega(2^{m})\)  \\
\optVI~\report{(\cref{alg:opt-vi})}                                           & \(\Theta(|\Ss_{f}|)\)               & \(\Theta(bm)\)     \\
\short{\nPI, \iPI} \report{\nPI~(\cref{alg:pi}), \iPI~(\cref{alg:better-pi})} & \(\Theta(|\p_{\text{min-safe}}|)\)  & \(O(bmn)\)         \\
\end{tabular}
\caption{
Best-case (where a safe policy exists) and worst-case runtimes of the algorithms we consider.
\report{We use \(m = |\A|\), \(n = |\Ss|\), branching factor of non-determinism \(b\), and \(|\p_{\text{min-safe}}|\) is a minimal safe policy's envelope size.}
}
\label{tab:algo-complexities}
\end{table}

\subsection{\TarjanSafe}\label{sec:tarjan-safe}

\TarjanSafe is JEA's algorithm for finding FOND state-avoiding policies.
\short{
It is a DFS with a labelling mechanism for states that have been proven safe or unsafe, and it detects SCCs similarly to Tarjan's algorithm~\cite{Tarjan1972}.
}
\report{
It is a DFS with a labelling mechanism for states that have been proven safe or unsafe, and it has a mechanism for detecting SCCs in its candidate policy using the stack and lowlink indices from Tarjan's algorithm~\cite{Tarjan1972}, hence the name \TarjanSafe.
We present pseudocode for the algorithm in \cref{alg:tarjansafe}.\footnote{JEA present a more general algorithm where a \textit{safety radius} \(r\) can be specified, for us \(r = \infty\). Also, we assume our state-avoiding formulation. Our presentation has been simplified accordingly.}
}
\TarjanSafe's DFS is defined recursively: in its recursive cases, \TarjanSafe iterates over \(\ac \in \A(\s)\) and calls itself on the successor states \(\transit(\s,\ac)\), then it either marks \s as unsafe if all these actions lead to unsafe states, or otherwise the algorithm assumes that \s is safe and proceeds.
Additionally, \TarjanSafe detects SCCs and marks the SCC's root as safe if the SCC contains no unsafe states.
Then, the recursion stops at \s without making further recursive calls in the following cases:
\short{
\red{(TODO: fix with update from report)}
\begin{enumerate*}[label=(\arabic*), itemjoin=\, , ref=\arabic*]
\item \s is a fail state (\(\s \in \failStates\)); \label{item:tarjansafe-s-fail-state}
\item \s is marked as unsafe; \label{item:tarjansafe-s-unsafe-label}
\item \s is a goal state (not relevant in our setting); \label{item:tarjansafe-s-goal}
\item \s is marked as safe; \label{item:tarjansafe-s-safe-label}
\item all actions in \(\A(\s)\) are unsafe; \label{item:tarjansafe-s-unsafe-actions}
\item \s is the root of an SCC with only safe states. \label{item:tarjansafe-s-scc}
\end{enumerate*}
For details, we refer the reader to JEA.
We motivate \TarjanSafe's practical effectiveness with its best-case runtime and show that its worst case is exponential.
}

\report{
\begin{enumerate}
\item \s is a fail state (\(\s \in \failStates\)), \label{item:tarjansafe-s-fail-state}
\item \s marked as unsafe with \(\texttt{known-unsafe}[\s]\), \label{item:tarjansafe-s-unsafe-label}
\item \s is a goal state (not relevant in our setting), \label{item:tarjansafe-s-goal}
\item \s marked as safe with \(\texttt{known-safe}[\s]\), \label{item:tarjansafe-s-safe-label}
\item some action in \(\A(\s)\) leads only to states in the stack (\(\transit(\s,\ac) \subseteq \texttt{stack}\)). \label{item:tarjansafe-s-leads-to-stack}
\end{enumerate}

\begin{algorithm}[ht]
  \small

\DontPrintSemicolon

\function{\(\TarjanSafe(\fond, \sZ)\)}{
  \(\texttt{known-safe}[\s] \gets \textit{false} \; \forall \s \in \Ss\)\;
  \(\texttt{known-unsafe}[\s] \gets \textit{false} \; \forall \s \in \Ss\)\;
  $\texttt{stack} \gets \text{empty stack}$\;
  $\texttt{low} \gets \text{empty map}$\;
  \Return $\textit{rec-\TarjanSafe}(\sZ)$\;
}

\function{\(\textit{rec-}\TarjanSafe(\s)\)}{
  \tcp{Has access to \TarjanSafe's vars}
  \If{$s \in \failStates$ \textbf{or} $\texttt{known-unsafe}[s]$} {
    \Return \textit{false} \;
  }
  \If{$\texttt{known-safe}[s]$}{
    \Return \textit{true} \;
  }
  $\texttt{stack-idx} \gets \|\texttt{stack}\|$ \;
  push \s onto \texttt{stack} \;
  $\texttt{low}[\s] \gets \texttt{stack-idx}$\;
  $\texttt{A-unsafe} \gets \textit{true}$\;
  \ForEach{$\ac \in \A(\s)$ \textbf{and while} $\texttt{A-unsafe}$}{
    $\texttt{a-safe} \gets \textit{true}$\;
    \ForEach{$\succState\in \transitions(\state, a)$ \textbf{and while} $\texttt{a-safe}$}{
        \If{$\succState\in \texttt{stack}$}{
          $\texttt{low}[\s] \gets \min(\texttt{low}[\s], \texttt{low}[\succState])$\;
        }
        \Else{
          $\texttt{a-safe}\gets \textit{rec-\TarjanSafe}(\succState)$\;
        }
      }
      $\texttt{A-unsafe} \gets \neg\texttt{a-safe}$\;
  }
  \If{\(\texttt{A-unsafe}\)}{
    $\texttt{known-unsafe}[s] \gets \textit{true}$
  }
  \Else{
    \tcp{Update \texttt{safe} at root of SCC}
    \If{$\texttt{low}[\s] = \texttt{stack-idx}$}{
      $\texttt{known-safe}[\s] \gets \textit{true}$\;
      pop from stack all states down to $\s$\;
    }
  }
  pop $\s$ from stack\;
  \Return \(\neg \texttt{A-unsafe}\)\;
}

  \caption{JEA's \TarjanSafe}
  \label{alg:tarjansafe}
\end{algorithm}

We motivate \TarjanSafe's practical effectiveness with its best-case runtime.
}

\begin{theorem} \label{thm:tarjansafe-best-case}
\TarjanSafe has a best-case runtime of \(\Theta(|\p_{\text{min-safe}}|)\) (if a safe policy exists).
\end{theorem}

\short{
\textit{Proof sketch.}
In the best case, \TarjanSafe will correctly guess \(\p_{\text{min-safe}}\) and traverse \(\p_{\text{min-safe}}\)'s policy graph once with its DFS to prove that all states are safe.
}

\report{
\begin{proof}
Suppose \(\p_{\text{min-safe}}\) is a safe policy with minimal policy graph such that the policy graph is a tree, except its leaf nodes which have self-loops.
Further suppose that \TarjanSafe guesses \(\p_{\text{min-safe}}\), i.e., in its DFS it always selects \(\p_{\text{min-safe}}(\s)\) for \s.
Then, the algorithm traverses \(\p_{\text{min-safe}}\) precisely once with its DFS to prove that the states are safe, and does not need to do any branching nor backtracking on the actions choices, so it is \(O(|\p_{\text{min-safe}}|)\).
To verify that a policy \p is safe, one must at least check that each state in its policy graph is not a fail state \failStates, so it is impossible to do better than \(\Omega(|\p_{\text{min-safe}}|)\).
Thus, \TarjanSafe's best-case must be \(\Theta(|\p_{\text{min-safe}}|)\).
\end{proof}
}

\report{
But in the worst case, \TarjanSafe is exponential.
}

\begin{theorem}
\TarjanSafe has a worst-case runtime of \(\Omega(2^m)\).
\end{theorem}
\begin{proof}
We present a class of problems where \TarjanSafe requires on the order of \(2^m\) steps, which proves that \TarjanSafe's worst-case is asymptotically \(2^m\) or worse, i.e., \(\Omega(2^m)\).
Consider the task in \cref{fig:tarjan-killer-2} with 5 layers, and consider the generalised construction with \(d\) layers.
This task has no fail, unsafe, nor goal states, so \TarjanSafe's cases \ref{item:tarjansafe-s-fail-state}--\ref{item:tarjansafe-s-goal} will never trigger for any state.
Case \ref{item:tarjansafe-s-safe-label} can only be triggered at an SCC's root, which is only \(\s_0\), so this case only triggers when the algorithm terminates.
Thus, the recursion can only bottom out at state \s if \s's applicable action leads to states that are already on the \(\texttt{stack}\), which in turn only happens at the state in the last layer, leading to \(\s_0\).
This forces the algorithm to re-expand states and enumerate all paths.
In our task with \(d\) layers, \TarjanSafe must enumerate all paths from \(\s_0\) to \(\s_{d+1}\) before it can terminate, which is \(2^d\) paths.
The problem has \(2d+2\) states and \(2d+2\) actions with a branching factor \(b\) of 2, so indeed the runtime of the algorithm is exponential \wrt the transition system size.
\end{proof}

\report{
We now show that with other approaches our problem can in fact be solved in polynomial time, even in the worst case.
}

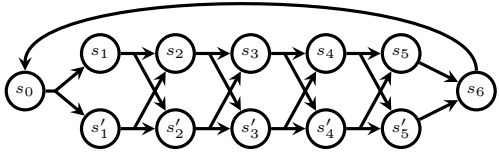
\begin{figure}
\center
\begin{tikzpicture}[>=stealth, node distance=1.5cm, very thick, minimum size=0.5cm, inner sep=0pt]
  \def\W{4}
  \def\H{1}

  \node[circle, draw, inner sep=2pt] (n-1) at (-1,0.5) {\tiny \(\s_0\)};
  \node[circle, draw, inner sep=2pt] (n-\the\numexpr\W+1\relax) at (\the\numexpr\W+1\relax,0.5) {\tiny \(\s_{\the\numexpr\W+2\relax}\)};

  \foreach \x in {0,...,\W}
    \node[circle, draw, inner sep=2pt] (n\x-0) at (\x,0) {\tiny $\raisebox{-1mm}{\smash{\(\s'_{\the\numexpr\x+1\relax}\)}}$};
  \foreach \x in {0,...,\W}
    \node[circle, draw, inner sep=2pt] (n\x-1) at (\x,1) {\tiny \(\s_{\the\numexpr\x+1\relax}\)};

  \path (n-1) ++(0.4,0.0) coordinate (aux-1);
  \draw[-] (n-1) -- (aux-1);
  \foreach \y in {0,...,\H}
    \draw[->] (aux-1) -- (n0-\y);

  \foreach \y in {0,...,\H}
    \draw[->] (n\W-\y) -- (n-\the\numexpr\W+1\relax);

  \foreach \x in {0,...,\numexpr\W-1}
    \foreach \y in {0,...,\H}
      \draw[->] (n\x-\y) -- (n\the\numexpr\x+1\relax-\y);

  \foreach \x in {0,...,\numexpr\W-1}
    \foreach \y in {0,...,\numexpr\H-1}
      {
        \path (n\x-\y) ++(0.45,0.0) coordinate (aux\x-\y);
        \draw[->] (aux\x-\y) -- (n\the\numexpr\x+1\relax-\the\numexpr\y+1\relax);
      }

  \foreach \x in {0,...,\numexpr\W-1}
    \foreach \y in {1,...,\numexpr\H}
      {
        \path (n\x-\y) ++(0.45,0.0) coordinate (aux\x-\y);
        \draw[->, out=0] (aux\x-\y) -- (n\the\numexpr\x+1\relax-\the\numexpr\y-1\relax);
      }

  \draw[->, out=90, in=90, looseness=0.5] (n-\the\numexpr\W+1\relax) to (n-1);

\end{tikzpicture}
\caption{Non-deterministic task with 5 layers and \(2^5\) paths from \(\s_0\) to \(\s_6\). Similarly constructed tasks with \(d\) layers have \(2^d\) paths from \(\s_0\) to \(\s_{d+1}\).}
\label{fig:tarjan-killer-2}
\end{figure}

\subsection{Unsafety Propagation}\label{sec:linear-algo}

\short{
\optVI solves state-avoiding tasks in linear time by propagating the unsafe region.
That is: (1) \optVI starts with a subset of the unsafe region \(\mathcal{Z} \gets \failStates\), (2) it selects \((\s, \ac)\) such that \(\transit(\s,\ac) \cap \mathcal{Z} \neq \emptyset\) and marks these \((\s, \ac)\) as unsafe, (3) if all actions in \(\A(\s)\) are unsafe then \s gets added to \(\mathcal{Z}\), (4) the algorithm iterates until no more \((\s, \ac)\) lead into \(\mathcal{Z}\).
Upon termination \(\mathcal{Z} = \mathcal{U}\), and thereby, a safe policy exists from \s iff \(\s \not\in \mathcal{Z}\).
\optVI adapts the linear-time algorithm of \citet{Diekert2022:reachability-games} for solving reachability games.

\begin{theorem}
\optVI has a best-case runtime of \(\Theta(|\failStates|)\) (if a safe policy exists) and a worst-case runtime of \(\Theta(bm)\).
\end{theorem}

\textit{Proof sketch.}
\optVI adds each \(\s \in \mathcal{U}\) to \(\mathcal{Z}\) precisely once.
\optVI's best case is \(\mathcal{U} = \failStates\), giving \(\Theta(|\failStates|)\).
The worst case must consider that \(\s \in \mathcal{U}\) may be considered multiple times if there are \(\ac \in \A(\s)\) so that \(\mathcal{T}(\s,\ac) \cap \mathcal{U} \neq \emptyset\) in multiple steps.
We take this into account with \(bm\): the number of actions \(m\) and their non-deterministic branching \(b\).

This algorithm has the lowest worst-case complexity of the algorithms we consider in this paper, and it seems unlikely that it is possible to do better.
However, in our setting, the unsafe region to propagate is very large (exponential \wrt problem description), making \optVI impractical.
}

\report{
We define \optVI, and show that it solves state-avoiding tasks in linear time by propagating the unsafe region.
This result draws from reachability games --- such games have two players who alternate turns: one player must reach certain states and their opponent must make sure that these states are never reached.
This framework is immediately analogous to our setting where the fail states \failStates correspond to the winning states of an ``unsafety player,'' and the ``safety player'' wins by ensuring that the unsafety player never wins --- a winning strategy for the safety player corresponds to a safe policy for us.
The unsafety player can be understood as an adversarial choice of non-deterministic outcomes.\footnote{FOND fairness is not violated with this interpretation, since the unsafety player will never voluntarily choose to get trapped in a cycle, as they are trying to reach unsafe states in finite steps.}
The existence of a linear-time algorithm is considered folklore for reachability games, e.g., \cite[Exercise~2.6]{Mazala2002:games}.
\citet{Diekert2022:reachability-games} present a concrete linear-time algorithm for solving reachability games.

In \cref{alg:opt-vi}, we define \optVI by adapting their algorithm to our state-avoiding FOND setting and formulating it in the framework of Value Iteration.
Effectively, \optVI is propagating the unsafe region from the known unsafe states (initially only the fail states), and then we know that a safe policy from \sZ exists iff \sZ is outside the unsafe region.
Thus, upon termination, \(\V(\s) = 1\) iff \s is unsafe.
Be aware that throughout execution the semantic of \V is slightly different: we have that \(\V(\s) = 1\) iff \s is proved unsafe, and \(\V(\s) = 0\) indicates that we have not proved \s either way.

\begin{algorithm}[]
{\small
\DontPrintSemicolon
  \caption{Propagate Unsafe States}\label{alg:opt-vi}
  \function{\(\optVI(\fond, \sZ)\)} {
    \tcp{Here \Q is a separate variable, not a function of \V}
    \(\V(\s) \gets \iverson{ \s \in \failStates } \; \forall \s \in \Ss\) \;
    \(\Qsa \gets 0 \; \forall \s \in \Ss, \ac \in \A(\s)\) \;
    \(\mathcal{Z} \gets \failStates\) \;
    \While{\(\mathcal{Z} \neq \emptyset\)}{
      \(\s' \gets \mathcal{Z}.\text{pop}()\) \;
      \For{\((\s, \ac) : \s' \in \mathcal{T}(\s,\ac)\)}{
        \(\Qsa \gets 1\) \;
        \If{\(\V(\s) = 0 \text{ and } \min_{\ac' \in \A(\s)} \Q(\s,\ac') = 1\)}{
          \(\V(\s) \gets 1\) \;
          \(\mathcal{Z} \gets \mathcal{Z} \cup \{\s\}\) \;
        }
      }
    }
    \Return \(\V(\sZ)\) \;
  }
}
\end{algorithm}

\begin{theorem}
\optVI~(\cref{alg:opt-vi}) returns \(0\) if \sZ is safe and \(1\) if \sZ is unsafe.
\end{theorem}
\begin{proof}
We want to show that a state \s updates \(\V(\s) \gets 1\) and enters \(\mathcal{Z}\) iff it is unsafe.
Only unsafe states enter \(\mathcal{Z}\) by an inductive argument: fail states \failStates are the base case, and otherwise \s only enters \(\mathcal{Z}\) if all its successors are unsafe with \(\min_{\ac' \in \A(\s)} \Q(\s,\ac') = 1\).
It remains to show that all unsafe states will eventually enter \(\mathcal{Z}\).
For contradiction, suppose that an unsafe state does not enter \(\mathcal{Z}\), which implies that there is an unsafe state \s with \(\V(\s) = 0\).
This state must have an action \(\ac \in \A(\s)\) that leads to an unsafe state \(\s'\) but has \(\Qsa = 0\).
Consequently, we can select an unsafe state \(\s' \in \mathcal{T}(\s,\ac)\) with \(\V(\s') = 0\).
Then we can apply the same argument to \(\s'\), and repeat this argument to produce a chain \(\langle \s_0, \s_1, \dots \rangle\) of unsafe states with \(\V(\s_i) = 0\) for all \(\s_i\) in the chain.
By selecting \textit{some action} with \(\Qsa = 0\) we are describing a policy, and since each state we are considering is unsafe there must be a path in this policy graph that leads to a fail state.
So, it must be possible to extend our chain so that it eventually reaches a fail state, but fail states have \(\V(\s) = 1\), yielding the desired contradiction.
\end{proof}

\begin{theorem}
\optVI~{(\cref{alg:opt-vi})} has a runtime of \(\Omega(|\mathcal{Y}|)\) and \(O(b |\mathcal{Y}|)\) where \(\mathcal{Y} = \{(\s,\ac) : \s \in \Ss, \ac \in \A(\s), \mathcal{T}(\s,\ac) \cap \mathcal{U} \neq \emptyset\}\), i.e., \(\mathcal{Y}\) are the state-action pairs with an effect in the unsafe region.
\end{theorem}

\begin{proof}
Each unsafe state \(\s \in \mathcal{U}\) will get added to \(\mathcal{Z}\) precisely once (as argued before), so the algorithm's inner \textit{for} loop will precisely iterate over \(\mathcal{Y}\) with some potential repetitions if \((\s, \ac) \in \mathcal{Y}\) has multiple unsafe states in \(\transit(\s,\ac)\).
This gives \(\Omega(|\mathcal{Y}|)\).
To handle the repetitions, we observe that \(\transit(\s,\ac)\) can not contain more unsafe states than the non-deterministic branching factor \(b\).
Thus, we iterate over \(\mathcal{Y}\) with at most \(b\) duplicates, i.e., the algorithm is \(O(b |\mathcal{Y}|)\).
\end{proof}

\begin{corollary}
\optVI~{(\cref{alg:opt-vi})} has a best-case runtime of \(\Theta(|\Ss_{f}|)\) (if a safe policy exists) and a worst-case runtime of \(\Theta(bm)\).
\end{corollary}

If no safe policy exists from the initial state \sZ then \optVI can stop as soon as \sZ enters \(\mathcal{Z}\), which may happen in a single step.

We call \optVI's worst-case runtime of \(\Theta(bm)\) linear, even though it contains two terms.
Recall that \(b\) corresponds to the non-deterministic branching factor, and \(m\) is the number of actions, so \(bm\) gives the number of transitions \(\langle \s, \ac, \s' \rangle\) in our state space.
In that sense, \optVI is linear \wrt the number of transitions.
Connecting to reachability games, \(bm\) corresponds to the number of edges in an equivalent game: the algorithm of \citet{Diekert2022:reachability-games} is linear over the number of edges in a game, so it is reasonable to consider our adaptation \optVI linear in our setting.

Since \optVI is in the VI framework, we can compare it to standard VI and variants like Topological VI~(TVI)~\cite{Dai2011:tvi}.
\optVI dominates VI and TVI in the sense that in its worst case it computes \(\Qsa\) once for each \(\s \in \Ss, \ac \in \A(\s)\), whereas VI and TVI perform the same computations in their best case, and usually more.

\optVI has the lowest worst-case complexity of the algorithms we consider in this paper, and it seems unlikely that it is possible to do better.
However, in our setting, the unsafe region and therefore \(\mathcal{Y}\) is very large (exponential \wrt problem description).
This makes \optVI impractical in practice.
It is possible to exploit problem structure and propagate the set of unsafe states with Binary Decision Diagrams (BDDs), as explained in \cref{sec:prop-unsafety-bdds}.
However, propagating unsafe states with BDDs still fails in our problems, because the BDDs explode in size after only a few steps of propagation.
It seems the approach of propagating unsafe states is fundamentally unsuitable to our problems.

In the next section, we present an alternative algorithm that combines the best of \TarjanSafe and \optVI: it can stop as soon as it finds a safe policy but maintains a polynomial worst-case runtime.
}

\short{
(TODO: shorten this and make it fit somewhere)
We call \optVI's worst-case runtime of \(\Theta(bm)\) linear, even though it contains two terms.
Recall that \(b\) corresponds to the non-deterministic branching factor, and \(m\) is the number of actions, so \(bm\) gives the number of transitions \(\langle \s, \ac, \s' \rangle\) in our state space.
In that sense, \optVI is linear \wrt the number of transitions.
Connecting to reachability games, \(bm\) corresponds to the number of edges in an equivalent game: the algorithm of \citet{Diekert2022:reachability-games} is linear over the number of edges in a game, so it is reasonable to consider our adaptation \optVI linear in our setting.
}

\subsection{Policy Iteration (PI)}\label{sec:policy-iteration}

Our implementation of PI in \cref{alg:pi}, named \nPI, makes a nice tradeoff: it has the same best case as \TarjanSafe and remains polynomial in the worst case.
PI is well-known as an MDP algorithm~\cite{Howard1960:pi}, and arguably, successful goal-reaching FOND algorithms such as PR2~\cite{Muise2024:PR2} fit within its framework.
In our setting, the idea of PI is that it starts with some policy \p, determines where \p is unsafe, fixes \p, and repeats until \p is safe or \sZ is proven unsafe.
We define \nPI using a Value Function \V with the semantic that \(\V(\s) = 1\) if we know that \s is unsafe, and \(\V(\s) = 0\) if we are unsure \report{(as in \optVI)}.
Initially, we only know that the fail states are unsafe so \(\V(\s) \gets \iverson{ \s \in \failStates }\)~(\cref{line:pi-v-init}).
We use the shorthand \(\Qsa = \max_{\s' \in \transit(\s,\ac)} \V(\s')\), which is \(1\) if we know it is unsafe to apply \ac in \s, and \(0\) if we are unsure.
Then, we propagate the states that we know are unsafe by updating \(\V(\s)\) with its minimal \qvalue~(\cref{line:pi-v-backup}).

\begin{algorithm}[]
{\small
\DontPrintSemicolon
  \caption{Na{\"i}ve Policy Iteration}\label{alg:pi}
  \function{\(\nPI(\fond, \sZ, \pInit)\)} {
    \(\V(\s) \gets \iverson{ \s \in \failStates} \; \forall \s \in \Ss\) \label{line:pi-v-init}\;
    \(\p \gets \pInit\) \;

    \RepeatForever{ \label{line:pi-loop-start}
      \(\V, \texttt{saw-unsafe} \gets \MarkUnsafe(\fond, \sZ, \V, \p)\) \;
      \If{\(\neg \texttt{saw-unsafe}\) or \(\V(\sZ) = 1\)}{
        \Return \(\V(\sZ)\) \;
      }
      \(\p \gets \Greedy(\fond, \sZ, \V)\) \;
    } \label{line:pi-loop-end}
  }

  \function{\(\MarkUnsafe(\fond, \sZ, \V, \p)\)}{
    \(\mathcal{E} \gets \text{DFS post order of policy graph \(\fond^{\p}\) from \sZ}\) \;
    \For{\(\s \in \mathcal{E}\)}{
      \( \V(\s) \gets \min_{\ac \in \A(\s)} \Qsa \) \label{line:pi-v-backup} \;
    }
    \(\texttt{saw-unsafe} \gets \iverson{ \exists \s \in \mathcal{E} \; \st \V(\s) = 1 }\) \;
    \Return \V, \texttt{saw-unsafe} \;
  }

  \function{\(\Greedy(\fond, \sZ, \V)\)} {
    new \p undefined everywhere \;
    \While{\p has undefined state reachable from \sZ} {
      \(\s \gets \) undefined state reachable from \sZ \;
      \(\p(\s) \gets \argmin_{\ac \in \A(\s)} \Qsa\) \;
    }
    \Return \p \;
  }
}
\end{algorithm}

We now analyse the runtime complexity of \nPI.

\begin{theorem}
\nPI\report{~(\cref{alg:pi})} has a best-case runtime of \(\Theta(\p_{\text{min-safe}})\) (if a safe policy exists).
\end{theorem}
\short{
\textit{Proof sketch.}
Similar argument as for \cref{thm:tarjansafe-best-case}.
}
\report{
\begin{proof}
The argument is similar to the proof of \cref{thm:tarjansafe-best-case}.
In the best case, \(\p_\text{init}\) is \(\p_{\text{min-safe}}\), so that \nPI only needs to perform a DFS over \(\p_\text{init}\) once to evaluate the policy.
\end{proof}
}

\report{
If no safe policy exists then \nPI can stop as soon as it propagates the unsafe region to \sZ and sets \(\V(\sZ) \gets 1\).
With an early-stop mechanism (as implemented later in \optPI), this may happen in as few as two steps.
}

\begin{theorem}
\nPI\report{~(\cref{alg:pi})} has a worst-case runtime of \(O(bmn)\).
\end{theorem}

\short{
\textit{Proof sketch.}
In each call to \MarkUnsafe, if \p remains unsafe, then at least one state with \(\V(\s) = 0\) gets updated with \(\V(\s) \gets 1\).
Thus, in the worst case, \nPI updates a single \(\V(\s)\) per iteration and then terminates.
This gives \(O(n)\) iterations with an \(O(bm)\) DFS in each pass.
}

\report{
\begin{proof}
The functions \MarkUnsafe and Greedy are both \(O(bm)\) because they are implemented with a variant of DFS that does not re-expand states, i.e., it expands each state at most once.
It remains to show that the main loop (lines \ref{line:pi-loop-start}--\ref{line:pi-loop-end}) repeats at most \(n\) times.
To this end we argue that for each pass of the loop where \p remains unsafe, at least one state with \(\V(\s) = 0\) gets updated with \(\V(s) \gets 1\); this can happen at most \(n\) times, because before then we either get a safe policy and terminate, or we set \(\V(\sZ) = 1\) and terminate.
Suppose we have computed \(\p \gets \Greedy(\fond, \sZ, \V)\) and \(\V, \text{\p-is-safe} \gets \MarkUnsafe(\fond, \sZ, \V, \p)\), and neither of the termination conditions are satisfied, i.e., \p is unsafe and \(\V(\sZ) = 0\).
We know that \(\V(\sZ) = 0\), since \p is unsafe its envelope must contain at least one state \(\s^{\dagger}\) with \(\V(\s^{\dagger}) = 1\), and there is a (potentially empty) sequence of intermediate states between \(\sZ\) and \(\s^{\dagger}\).
On that sequence, there must be a pair of states \(\s_i\) with \(\V(\s_i) = 0\) and \(\V(\s_{i+1}) = 1\) by a discrete intermediate value argument, i.e., if \(\V(\s)\) can only be \(0\) or \(1\), and at one end of the sequence we have \(\V(\sZ) = 0\) and at the other \(\V(\s^{\dagger}) = 1\), then somewhere it has to switch from \(0\) to \(1\).
Consequently, \(\Q \big( \s_i, \p(\s_i) \big) = 1\).
Recall that \p was constructed greedily, so \(\min_{\ac \in \A(\s_i)} \Qsa = 1\) which means that \(\V(\s_i)\) must be updated to \(1\), concluding the proof.
\end{proof}
}

The implementation in \cref{alg:pi} has various inefficiencies which we address in the \textbf{Improved Policy Iteration} algorithm \optPI\report{~(\cref{alg:better-pi})}.
The key change in \optPI is that the construction of a greedy policy and its evaluation are combined into a single step.
Consequently, the algorithm only needs to traverse the candidate policy's envelope once per iteration, rather than twice.
Moreover, it enables us to stop the construction of the greedy policy as soon as an unsafe state has been detected in its envelope\report{~(\cref{alg:better-pi}~\cref{line:rec-pi-check-unsafe-return})}; this avoids the expensive situation of \nPI, where \(\Greedy(\!\cdots\!)\) encounters an unsafe state early on but continues construction, even though the policy can not be safe.
\short{
\iPI has the same complexities as \nPI (by similar arguments), and preliminary results showed \iPI to be significantly faster in practice.
}

\report{
In the change from \nPI to \iPI we also generalise the notion of an initial policy into an action order \actionOrder where \(\actionOrder(\s)\) returns the same actions as \(\A(\s)\) but ordered.
Consequently, we do not provide \(\pInit\), as it is captured by the first ``preference'' of each \(\actionOrder(\s)\).
This leads to the subtle detail that we no longer construct policies that are greedy \wrt \V, but rather we construct policies according to the preferences of \actionOrder.
We made this choice because profiling revealed that computing \qvalues to compute the greedy policy \wrt \V is relatively expensive, and does not pay off, i.e., using an arbitrary fixed order over actions had better performance.
The choice of \actionOrder can be important, since a good ordering may quickly lead to a safe policy or to a way to prove \sZ as unsafe, and a bad one causes \optPI to waste time.

We note that with \iPI's recursive definition it is not immediately clear whether the main loop (lines \ref{line:better-pi-loop-start}--\ref{line:better-pi-loop-end}) is still necessary.
Indeed, it is still necessary, as demonstrated by the example in \cref{sec:ipi-main-loop-necessary}.

\iPI~(\cref{alg:better-pi}) has the same best-case and worst-case runtimes as \nPI~(\cref{alg:pi}), by similar arguments.

\begin{algorithm}[ht]
{\small
\DontPrintSemicolon
  \caption{Improved Policy Iteration}\label{alg:better-pi}
  \function{\(\optPI(\fond, \sZ, \text{ action order } \actionOrder)\)} {
    \(\V(\s) \gets \iverson{ \s \in \failStates} \; \forall \s \in \Ss\) \;
    \RepeatForever{ \label{line:better-pi-loop-start}
      \(\texttt{seen} \gets \emptyset\) \;
      \(\texttt{saw-unsafe} \gets \text{false}\) \;
      \(\reciPI(\sZ)\) \;
      \If{\(\neg \texttt{saw-unsafe}\) or \(\V(\sZ) = 1\)} {
        \Return \(\V(\sZ)\) \;
      }
    } \label{line:better-pi-loop-end}
  }

  \function{\(\reciPI(\s)\)}{
    \tcp{Rec-iPI has access to \optPI's vars}

    \If{\(\V(\s) = 1\)}{
      \(\texttt{saw-unsafe} \gets \text{true}\) \;
      \Return unsafe \; \label{line:rec-pi-check-unsafe-return}
    }
    \If{\(\s \in \texttt{seen}\)}{
      \Return maybe-safe \;
    }
    \(\texttt{seen} \gets \texttt{seen} \cup \{\s\}\) \;

    \tcp{Look for safe action in \(\actionOrder(\s)\)}
    \For{\(\ac \in \actionOrder(\s)\)} {
      \For{\(\s' \in \supp(\s,\ac)\)}{ \label{line:rec-pi-check-safe-start}
        \(\s'\texttt{-status} \gets \reciPI(\s')\) \;
      }
      \If{\(\s'\texttt{-status} \neq \text{unsafe} \; \forall \s' \in \supp(\s,\ac)\)}{
        \Return maybe-safe \;
      }
    }

    \tcp{All actions unsafe means \s unsafe}
    \(\V(\s) \gets 1\) \;
    \(\texttt{saw-unsafe} \gets \text{true}\) \;
    \Return unsafe \;
  }
}
\end{algorithm}

}

\short{
Note that \nPI requires an \textbf{initial policy} \(\pInit\); \optPI generalises this and takes an \textbf{action ordering} for each state, i.e., a function that determines in which order the actions in \(\A(\s)\) are explored.
This choice can be important, since a good ordering may quickly lead to a safe policy or to a way to prove \sZ as unsafe, and a bad one causes \optPI to waste time.
}

\begin{figure*}[ht!]
\centering

\includegraphics[width=\textwidth]{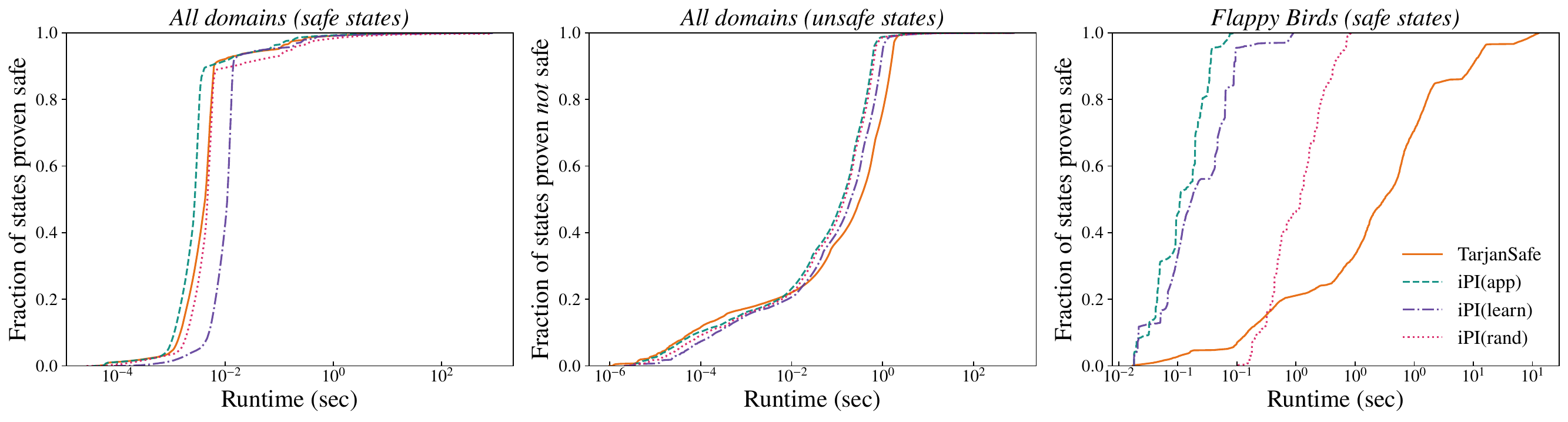}

\caption{
Cumulative plots of the proportion of states decided \wrt time.
The plots separate safe and unsafe states: (left) is an aggregation over all domains' safe states, (middle) is over all unsafe states, and (right) is over safe states in Flappy Bird.
}
\label{plot:short-yes-no-summary}
\end{figure*}

\section{Experiments}\label{sec:experiments}
We empirically compare \optPI with JEA's \TarjanSafe for deciding state safety: given a learned policy \p that we wish to check for safety, we find a state \s in \p's graph, and interrogate \optPI and \TarjanSafe whether \s is safe.
Recall that finding such states is at the core of JEA's pipeline for testing the safety of the learned policy.
We consider \optPI with three different action orderings (1) \orderApp: the order of the model (same as \TarjanSafe); (2) \orderLearn: using the learned action policy that we are evaluating; (3) \orderRand: randomised.
\short{We do not consider \optVI because it runs out of memory.}
\report{We do not consider \optVI because it runs out of memory (see \cref{sec:linear-algo} and \cref{sec:prop-unsafety-bdds}).}
Our implementation extends JEA's C++ code and is available at~\cite{schmalz:jain:icaps2026-zenodo}.
All experiments were run on Intel Xeon E5-2660 CPUs with 12GB memory and a timeout of 15 minutes.

\paragraph{Benchmarks.}
We use JEA's benchmarks (scaled up), consisting of non-deterministic variants of \emph{blocksworld} and \emph{transport}; deterministic control benchmarks (\emph{bouncing ball}, \emph{follow car} %
and \emph{inverted pendulum}); as well as different variants of a transport-like domain called \emph{one/two-way line}, where a truck moves uni-(respectively bi-)directionally on a line, carrying packages to the right.
We also introduce a new benchmark, \emph{Flappy Bird}.
\short{
A safe policy directs the bird around a repeating track indefinitely while avoiding stationary obstacles.
}
\report{
In this problem, a bird must navigate a 2D environment from left to right, avoiding stationary obstacles.
It can move up or down, and in each move may non-deterministically move to the right.
Once the bird reaches the right side of the environment, it gets reset to the left side, i.e., it moves through the environment indefinitely unless it hits an obstacle.
Thus, a safe policy moves the bird through the environment indefinitely, avoiding all obstacles.
}
When safe policies exist, this showcases \TarjanSafe's bad performance, taking inspiration from the example in \cref{fig:tarjan-killer-2} in a more practical setting.
Following JEA, all problems are encoded in JANI~\cite{Budde2017:JANI}.
For learned policies to test, we trained neural networks via $Q$-learning with two hidden layers of size 64.

\paragraph{Results.}
Our results are summarised in \cref{plot:short-yes-no-summary}.
We note the benchmarks are solved in small timescales: the \optPI variants solve almost all problems within 1~sec.
This is not because the problems are small: we handle instances with up to 80 integer variables bounded by 70.
To motivate this further, \optVI runs out of memory as it enumerates all exponentially-many fail states (\wrt JANI description); similarly, \optVI with regression over Binary Decision Diagram~(BDDs), which can represent a set of states compactly, also runs out of memory as the BDDs become too large.
Rather, the benchmarks are solved quickly by exploiting the features of state-avoiding problems.
For deciding safe states, there are often cycles that \TarjanSafe and \optPI use to quickly construct safe policies that avoid fail states.
For deciding unsafe states, \TarjanSafe and \optPI do not need to explore deeply before finding a way to propagate unsafety to \sZ.

\paragraph{Comparing the variants of \optPI:} \iPIApp dominates the others.
The orderings themselves do not appear to have a significant impact: all orderings tend to guess safe policies quickly (for safe states) and propagate the unsafe region to \sZ efficiently (for unsafe states).
Rather, the performance difference comes from implementation, namely the overhead of sorting with \textit{learn} and randomising with \textit{rand}.

\paragraph{Comparing \optPI to \TarjanSafe:} For deciding states over all domains (left and middle), there is no significant difference between the approaches.
JEA's benchmarks contribute most to these results, and they tend to be amenable to \TarjanSafe, so this demonstrates that \optPI indeed has similar best-case behaviour to \TarjanSafe.
In contrast, for deciding safe states in Flappy Birds, \optPI scales exponentially better than \TarjanSafe (note the runtime scale is logarithmic).
This confirms our theory that \optPI is exponentially faster than \TarjanSafe in its worst case.
So, \iPI never performs significantly worse than \TarjanSafe, and in some cases performs exponentially better.

\paragraph{Summary.}
These results confirm our theory that \iPIApp dominates \TarjanSafe in the sense that it is similar for problems amenable to \TarjanSafe, and exponentially better in problems that are problematic for \TarjanSafe such as Flappy Bird.
Furthermore, the results show that within \iPI, the default action ordering performs best in our setting.
We give a breakdown over domains and more discussion in \cref{sec:detailed-results}, these are consistent with our findings here.

\section{Conclusion}\label{sec:conclusion}

The state-safety decision problem is significant, being a necessary component of JEA's safety assurance pipeline for learned policies.
We considered JEA's \TarjanSafe algorithm for deciding state safety, and introduced \optVI and \optPI.
We showed \TarjanSafe has an exponential worst-case runtime but a good best case; in contrast, \optVI has a linear worst case but impractical best case;
\optPI fills the gap, combining a polynomial worst case while maintaining the same best case as \TarjanSafe.
We confirmed this empirically: \optPI is comparable to \TarjanSafe when \TarjanSafe does well, and otherwise \optPI scales exponentially better.
This places \optPI as the current best algorithm for deciding state safety in our framework, and as the most promising starting point for future work.

\section*{Acknowledgements}

Funded by the Deutsche Forschungsgemeinschaft (DFG, German Research Foundation) -- GRK 2853/1 ``Neuroexplicit Models of Language, Vision, and Action'' - project number 471607914.

\appendix %

\report{
\section{Propagating Unsafety with Binary Decision Diagrams}\label{sec:prop-unsafety-bdds}
Rather than regressing unsafety directly in the state space, it is natural to try to exploit problem structure and regress using Binary Decision Diagrams (BDDs).
In JEA's setting, the set of fail states $S_f$ is compactly represented as a symbolic formula $\phi_F$ called the \textbf{fail condition}.
The symbolic representation of the unsafe region is equivalent to the weakest precondition a state has to satisfy to guarantee reaching the fail condition.
This can be computed using regression~\cite{jussi:rintanen:ecai-08} and BDDs are state-of-the-art structures for such computations.
A concrete implementation is provided in \cref{alg:unsafe_region_BDD}.\footnote{BDDs are restricted to boolean variables, but our transition system deals with linear constraints on integer variables, so we encode the constraints in \textit{toBDD} as in \citet{bartzis:bultan:sttt-06} .}

This algorithm uses $\rho$ to represent the previously found unsafe region and $\tau$ for the newly found unsafe states. In each iteration, $\rho$ is updated by adding the states in $\tau$.
$\tau$ is found by computing the preimage of each action with respect to $\rho$.
Here, each action consists of a \textbf{guard} (g): to represent the weakest precondition a state must satisfy for the action to be applicable in it; and \textbf{updates} ($u_1, \dots u_n$): to represent the nondeterministic effects of applying an action on a state.
The following describes the preimage computation for a given set of successor states $S(v')$ and an action $a$. The set of predecessor states that can reach $S(v')$ on applying $a$ is: $g(v) \land \exists v' (S(v')\land (u_1(v,v') \lor \dots u_n(v, v')))$.
Intuitively, $\exists v'. S(v') \land (u_1(v, v') \lor\dots u_n(v,v'))$ is a relational product~\cite{burch:etal:tcad-94} to represent states where applying $a$ leads to states in $S(v')$. A disjunctive partitioning ensures any nondeterministic effect leading to $S(v')$ is included. Conjuncting with $g(v)$ filters out all states where this action is not applicable.

\begin{algorithm}[ht]
  \small
    \SetKwFunction{tro}{TR}
    \SetKwProg{Fn}{Function}{$\to$ BDD:}{}
    \SetKwFunction{Regr}{Reg}
    \SetKwFunction{final}{computeUnsafe}
    \SetKwFunction{toBDD}{toBDD}

    \Fn{\final{$\phi$}}{
        $\rho\gets$ \toBDD$(\bot); \tau\gets$ \toBDD$(\phi)$\;
        \While{$\tau\not\subseteq \rho$}{
            $\rho\gets \rho \lor \tau$\;
            $\tau\gets$ \toBDD~$(\top)$\;
            $G \gets$ \toBDD~$(\bot)$\;
            \ForEach{$a\in A$}{
                $\tau \gets \tau~\land \tro(a, \rho)$\;
            }
        }
        \KwRet{$\rho$}\;
    }

    \Fn{\tro{$a, \rho$}}{
        \uIf{$a = \langle g,u_0|\dots|u_n\rangle$}{\KwRet{$\toBDD(g) \land (RelProd(\toBDD(u_0), \rho)\lor\dots\lor RelProd(\toBDD(u_n), \rho))$} }
    }
    \caption{Compute the unsafe region using Binary Decision Diagrams}
    \label{alg:unsafe_region_BDD}
\end{algorithm}

\textbf{\optVI is not practical for our problems.}
The fail condition \(\phi_F\) describes exponentially-many fail states \failStates, so \cref{alg:opt-vi} runs out of memory in its first step of enumerating fail states, even before attempting to propagate the unsafe region.
With BDDs (\cref{alg:unsafe_region_BDD}), it is possible to represent \(\phi_F\) compactly, but after propagating it a few steps to describe the unsafe region the BDDs explode in size, again running out of memory.
Thus, it seems that regressing the unsafe region is not a suitable approach to solve our problems.

\section{Is the Main Loop in \iPI Necessary?}\label{sec:ipi-main-loop-necessary}

\textit{Is the main loop~(lines \ref{line:better-pi-loop-start}--\ref{line:better-pi-loop-end}) necessary, or is a single call to \(\reciPI(\sZ)\) enough?}
The loop is indeed necessary, because there are cases where \(\reciPI(\sZ)\) constructs an unsafe policy \p and does not return unsafe.
Let us apply \optPI to the state-avoiding task shown in \cref{fig:rec-pi-finds-unsafe}, the key steps are:
\begin{itemize}
\item suppose \(\reciPI\) first traverses \(\sZ, \ac_0, \s_1, \ac_1, \s_2, \ac_2\) and then encounters \(\s_2\) which is marked as \texttt{seen}, so the recursion unwinds to \(\s_1\)
\item at \(\s_1\), \(\reciPI\) sees that the only action \(\ac_1\) inevitably encounters a fail state, so it sets \(\V(\s_1) = 1\) and backtracks to \sZ,
\item at \sZ \(\reciPI\) selects \(\ac'_0\) as a potentially safe alternative to \(\ac_0\) --- this leads to \(\s_2\) which is in \texttt{seen}, now the recursion fully unwinds.
\end{itemize}
However, the policy \p with \(\p(\sZ) = \ac'_0\), \(\p(\s_1) = \ac_1\), and \(\p(\s_2) = \ac_2\) is unsafe.
The key insight is that \(\reciPI\) only propagates that states are unsafe with \(\V(\s) = 1\) backwards, but in the presence of cycles, as we see in \cref{fig:rec-pi-finds-unsafe}, it may also be necessary to propagate this information forwards to detect that the policy is unsafe.

\begin{figure}[ht]
\center
\begin{tikzpicture}[>=stealth, node distance=1.5cm, very thick, minimum size=0.5cm]
  \node[circle, draw, inner sep=2pt] (n0) at (0,0) {\sZ};
  \node[circle, draw, inner sep=2pt] (n1) at (3,0) {\(\s_1\)};
  \node[circle, draw, inner sep=2pt] (n2) at (6,0) {\(\s_2\)};
  \node[circle, draw, inner sep=2pt] (nX) at (4.5,0.7) {\(\times\)};

  \draw[-] (n0) -- node[below] {\(\ac_0\)} (n1) ;;

  \draw[->] (n1) -- node[below] {\(\ac_1\)} (n2) ;
  \coordinate (aux1) at (4, 0);
  \draw[->, out=45, in=180+45, looseness=1] (aux1) to (nX);

  \draw[->, out=90, in=90] (n2) to node[above] {\(\ac_2\)} (n1);

  \draw[->, out=-45, in=180+45, looseness=0.4] (n0) to node[below] {\(\ac'_0\)} (n2);

\end{tikzpicture}
\caption{A state-avoiding task where a single call to \reciPI finds an unsafe policy but does not recognise it as unsafe.}
\label{fig:rec-pi-finds-unsafe}
\end{figure}
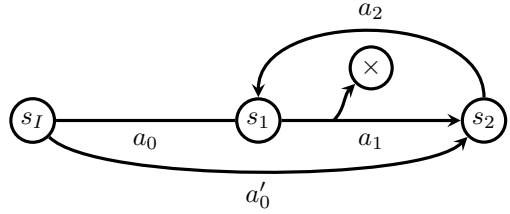

\section{Detailed Results}\label{sec:detailed-results}
We break down the plots from \cref{plot:short-yes-no-summary} by their domains in \cref{plot:no-extended} and \cref{plot:yes-extended}.
Over domains, the results are consistent with the aggregated results in \cref{sec:experiments}, i.e., \iPIApp and \TarjanSafe have similar performance on JEA's benchmarks and an exponential difference when deciding safe states on Flappy Birds; within \iPI, the action ordering \orderApp generally outperforms \orderLearn and \orderLearn.

We emphasise that \TarjanSafe's worst case is only triggered by Flappy Bird when deciding \emph{safe states}, and not when deciding \emph{unsafe states}.
This is because, for unsafe states, \sZ can be proven unsafe before all paths need to be enumerated, bringing the task closer to \TarjanSafe's best case.
Consequently, there is no significant difference between \iPIApp and \TarjanSafe when deciding unsafe states for Flappy Bird, and this does not contradict any of our claims.

\begin{figure*}[t!]
\centering

\begin{tikzpicture}
\node[
  draw,
  thick,
  inner sep=8pt,
  align=center
] (unsafebox) {
  \begin{minipage}{0.99\textwidth}
    \centering
    \includegraphics[scale=0.51, valign=t]{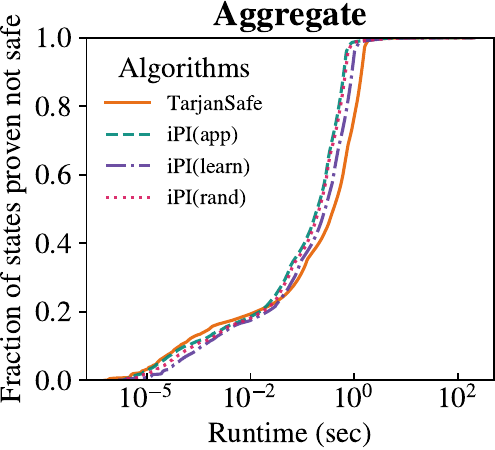}
    \includegraphics[scale=0.51, valign=t]{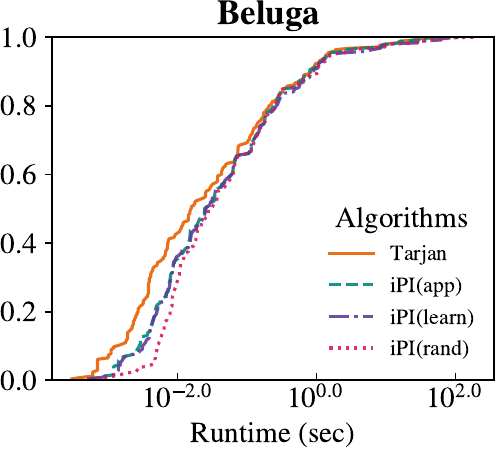}
    \includegraphics[scale=0.51, valign=t]{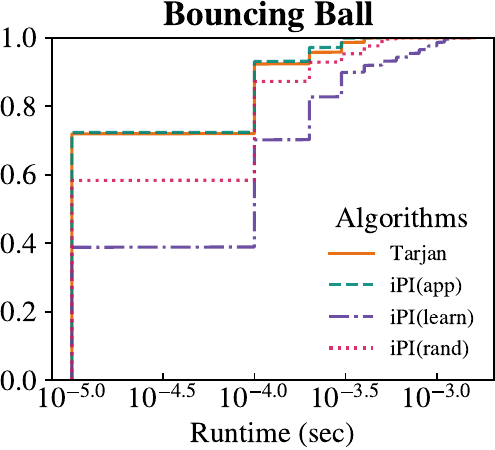}
    \includegraphics[scale=0.51, valign=t]{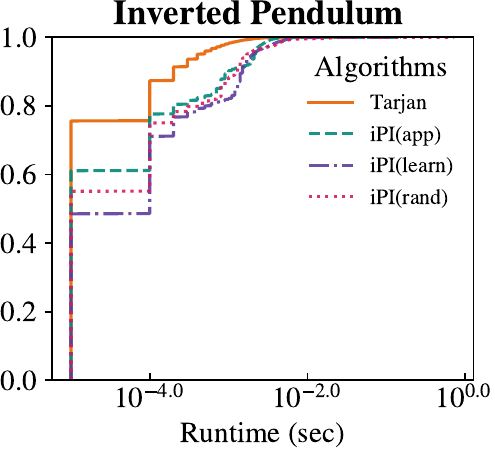}
    \includegraphics[scale=0.51, valign=t]{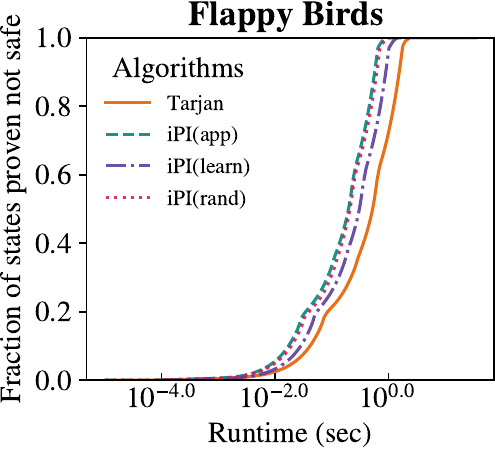}
    \includegraphics[scale=0.51, valign=t]{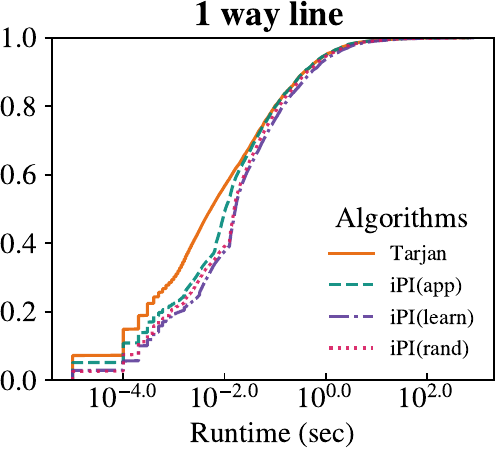}
    \includegraphics[scale=0.51, valign=t]{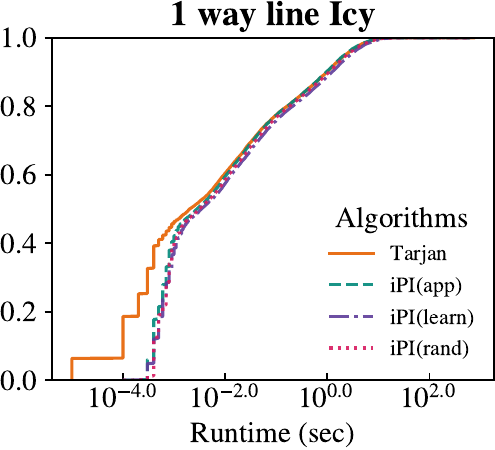}
    \includegraphics[scale=0.51, valign=t]{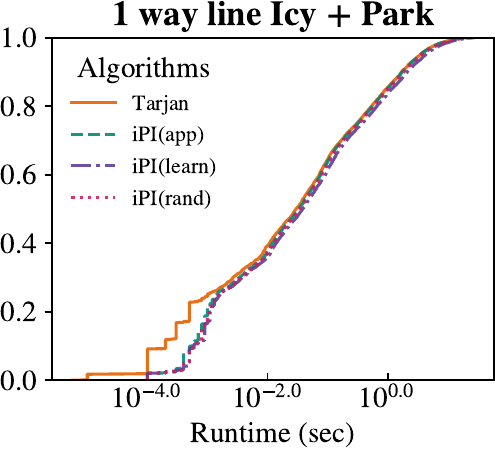}
    \includegraphics[scale=0.51, valign=t]{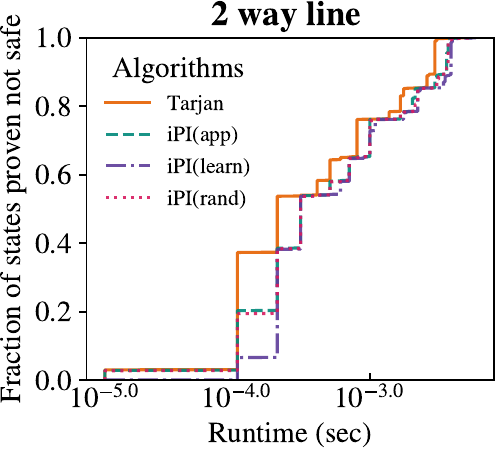}
    \includegraphics[scale=0.51, valign=t]{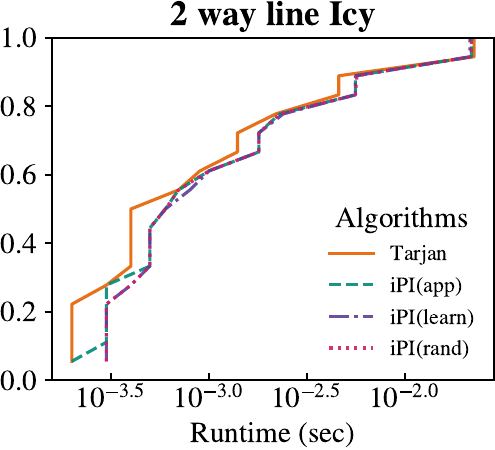}
    \includegraphics[scale=0.51, valign=t]{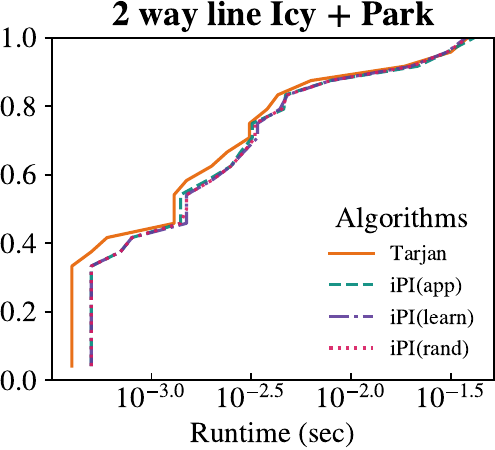}

  \end{minipage}
};
\end{tikzpicture}

\caption{
Cumulative plots of the proportion of states decided \wrt time for \textbf{unsafe states}.
}
\label{plot:no-extended}
\end{figure*}

\begin{figure*}[t!]
\centering

\begin{tikzpicture}
\node[
  draw,
  thick,
  inner sep=8pt,
  align=center
] (safebox) {
  \begin{minipage}{0.99\textwidth}
    \centering
    \includegraphics[scale=0.51, valign=t]{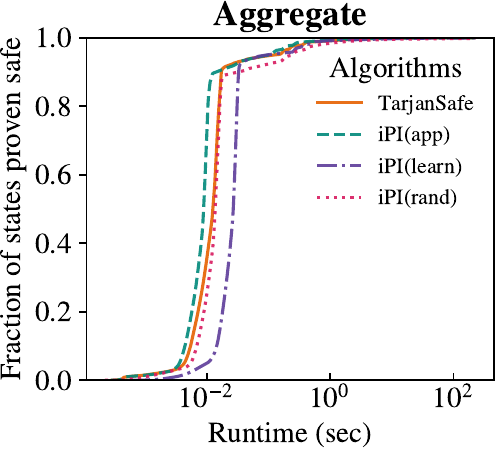}
    \includegraphics[scale=0.51, valign=t]{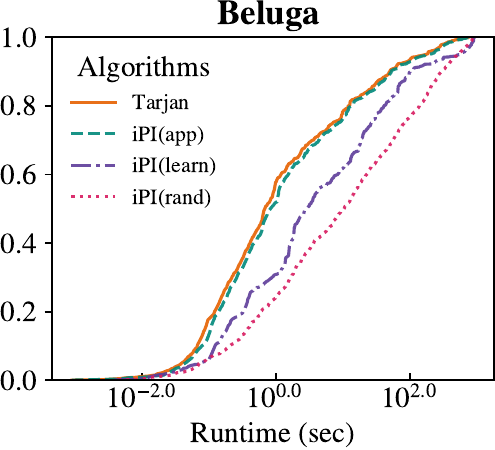}
    \includegraphics[scale=0.51, valign=t]{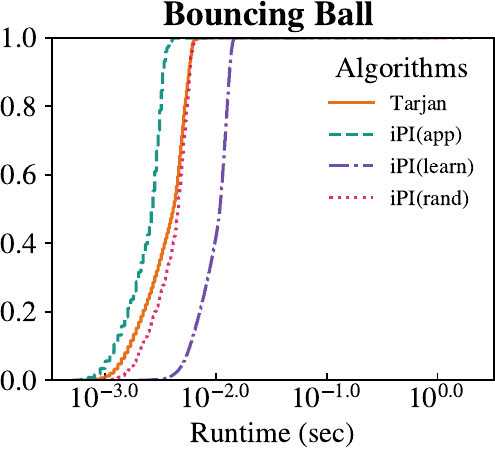}
    \includegraphics[scale=0.51, valign=t]{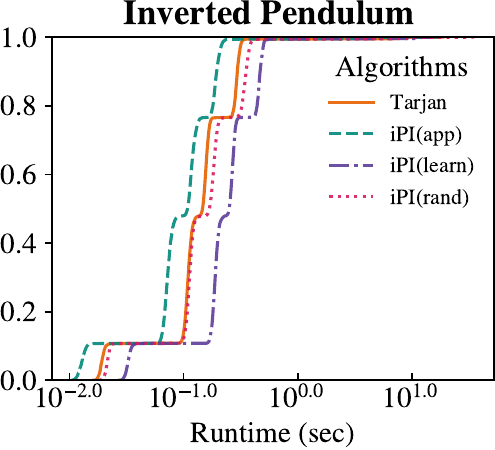}
    \includegraphics[scale=0.51, valign=t]{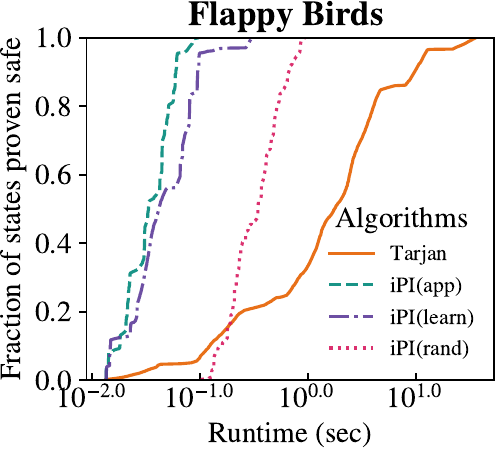}
    \includegraphics[scale=0.51, valign=t]{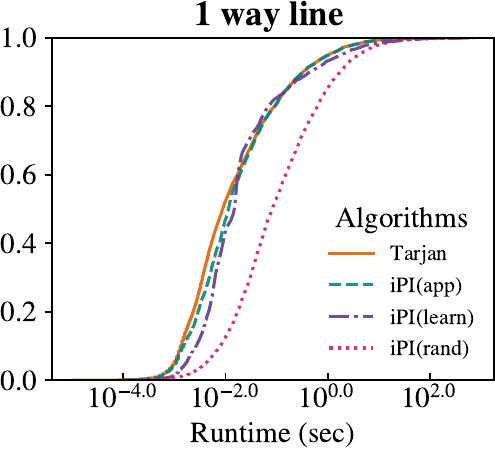}
    \includegraphics[scale=0.51, valign=t]{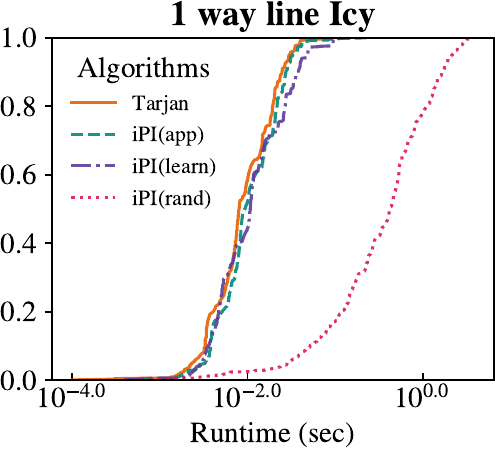}
    \includegraphics[scale=0.51, valign=t]{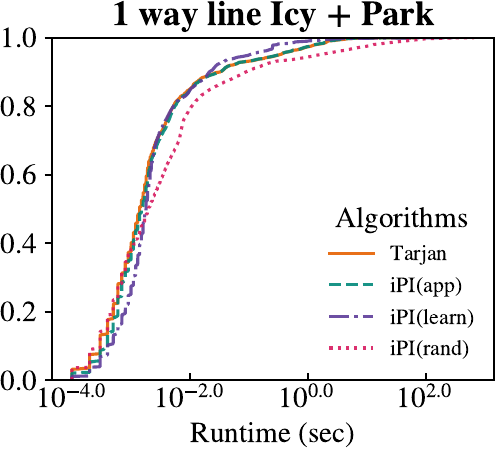}
    \includegraphics[scale=0.51, valign=t]{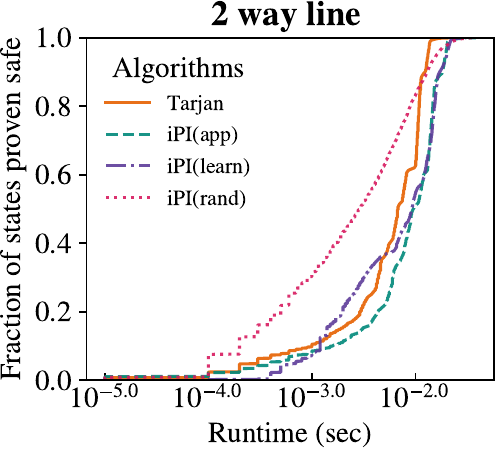}
    \includegraphics[scale=0.51, valign=t]{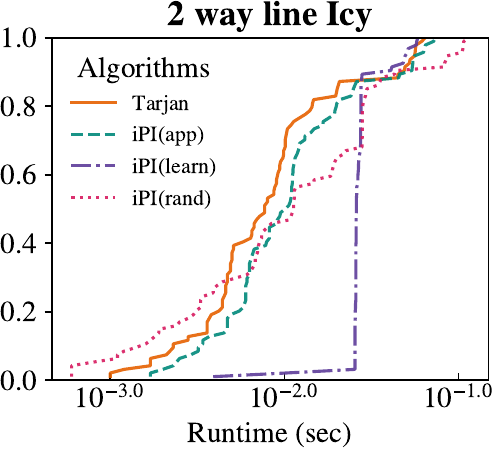}
    \includegraphics[scale=0.51, valign=t]{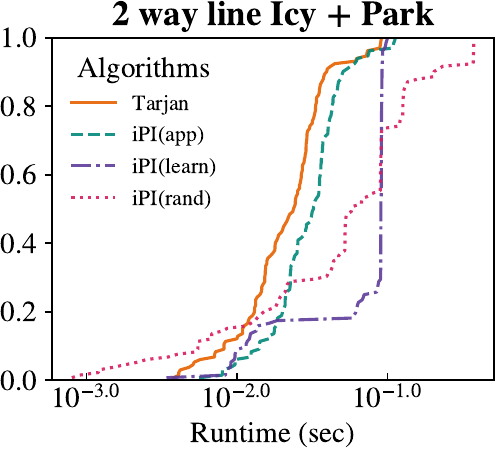}
  \end{minipage}
};
\end{tikzpicture}

\caption{
Cumulative plots of the proportion of states decided \wrt time for \textbf{safe states}.
}
\label{plot:yes-extended}
\end{figure*}

}

\report{
\clearpage
\clearpage
}

\bibliography{references}

\end{document}